\documentclass[conference,letterpaper]{IEEEtran}

\usepackage{times}
\usepackage{amsmath,amssymb,amsthm,amsfonts}
\usepackage{algorithm,algorithmic}
\usepackage{graphicx}
\usepackage{hyperref}
\usepackage{enumitem}
\usepackage[numbers,sort&compress]{natbib}
\usepackage{booktabs}
\usepackage{caption}
\usepackage[caption=false,font=footnotesize]{subfig}
\usepackage{balance}

\newtheorem{theorem}{Theorem}
\newtheorem{proposition}[theorem]{Proposition}
\newtheorem{lemma}[theorem]{Lemma}

\newtheorem{definition}[theorem]{Definition}
\newtheorem{remark}[theorem]{Remark}
\newtheorem{assumption}[theorem]{Assumption}

\newcommand{\E}{\mathbb{E}}
\newcommand{\Prob}{\mathbb{P}}

\newcommand{\cH}{\mathcal{H}}
\newcommand{\cX}{\mathcal{X}}

\newcommand{\cS}{\mathcal{S}}

\title{Delay Complexity Dimension:\\Learning with Delayed and Censored Labels}
\author{\IEEEauthorblockN{Madhava Gaikwad}
\IEEEauthorblockA{Independent Researcher\\
gaikwad.madhav@gmail.com}}
\date{}

\begin{document}
\maketitle

\begin{abstract}
We study PAC learning of binary classifiers when labels arrive after
random delays and are censored beyond an observation window~$W$. How
many samples does this observation model require? We propose a
distribution-dependent answer, the \emph{Delay Complexity Dimension}
(DCD). The DCD combines the VC dimension with the exact
label-information rate of the observation channel. Every learner
needs $\Omega(\mathrm{DCD}/\varepsilon^2)$ samples. An oracle
pseudo-label ERM matches this rate on the instance family of the
lower bound, up to constants and logarithmic factors. On that family
the DCD is within a constant factor of the classical censored-ERM
complexity, and whether the DCD rate is achievable where the two
separate remains open. The theory also yields a quantization scheme
under which delay information lowers the residual label variance, an
expedition policy with a $(1-1/e)$ guarantee for its version-space
objective, and a phenomenon with no analogue in the no-delay setting.
Under a fixed collection horizon, a moderate window can strictly beat
immediate labels. Experiments on synthetic data and the Criteo
conversion logs support these predictions. A budget-matched test of a
companion labeler-diversity prediction is reported as a negative
result.
\end{abstract}

\section{Introduction}
\label{sec:intro}

In many classification problems the labels are \emph{delayed}. Ground
truth arrives some time after the feature vector is observed, and it
is censored if a fixed observation window~$W$ expires first.
Conversion-rate prediction~\citep{chapelle2014modeling}, fraud
detection, clinical event prediction, and any problem with
administrative censoring share this structure. The standard view
treats delay as a reduction in effective sample size. It scales the
complexity by the probability $q(W)=\Prob(D\le W)$ that the delay
$D$ does not exceed the window. This view is incomplete. The delay itself carries
information. A user who has not converted by day~7 is less likely to
convert at all. A patient whose biomarker has not crossed a threshold
by month~3 is less likely to do so eventually.

This intuition has been established in two parallel literatures:
\citet{yang2022generalized} formalized it for recommender systems
via conditional entropy, and \citet{masoudian2023delayed}
characterized when intermediate observations help in delayed
bandits. The corresponding question for batch supervised
classification has remained open. Our work is concerned with the
following question: \emph{how many samples does PAC learning
require when labels are delayed and censored?}

We address this question by introducing the Delay Complexity
Dimension and proving upper and lower bounds in terms of it. The
bounds match on the instance family of the lower-bound construction.
The rest of the paper develops this answer and delimits its scope.

\textbf{Contributions.}
\begin{enumerate}[itemsep=1pt,topsep=2pt,leftmargin=*]
\item \textbf{DCD:} a distribution-dependent dimension
(Definition~\ref{def:dcd}). It combines $\mathrm{VC}(\cH)$ with an
effective coverage $\eta(W)$, the per-sample label-information rate
of the observation channel.
\item \textbf{Near-matching PAC bounds.} Every learner needs
$n=\Omega\bigl((\mathrm{DCD}+\log(1/\delta))/\varepsilon^2\bigr)$
samples (Theorem~\ref{thm:dcd-lb}). An oracle pseudo-label ERM
attains this rate on every instance family of the lower-bound
construction (Theorem~\ref{thm:dcd-ub}). On that family the DCD is
within a constant factor of the censored-ERM complexity $d/q(W)$
(Remark~\ref{rem:balanced}), so the characterization is established
only where the two agree. The regime where they separate is open,
and we conjecture the match extends to it
(Remark~\ref{rem:tightness}).
\item \textbf{Quantization.} A constructive quantization lets
pseudo-label ERM use the delay value itself. Delay information lowers
the residual label variance at no cost in the oracle complexity term,
and the effective sample size is $n_{\mathrm{eff}}=c\,n\,q(W)$
(Lemma~\ref{lem:quantize}, Theorems~\ref{thm:info-tradeoff}
and~\ref{thm:neff}). The cost of finer bins is the calibration of
more bin posteriors (Remark~\ref{rem:plugin}).
\item \textbf{Active expedition.} Uniform-random expedition of $B$
labels raises the effective sample size to $n\,q(W)+B$
(Theorem~\ref{thm:expedition}). A threshold policy on the product of
version-space disagreement and censoring probability,
$\Delta(x)\cdot\Prob(D>W\mid x)$, achieves a $(1-1/e)$ approximation
for the version-space objective (Theorem~\ref{thm:policy}); its
excess-risk analysis is open (Remark~\ref{rem:scope}).
\item \textbf{Structural predictions.} Under a fixed collection
horizon, a moderate window $W^\star>0$ can strictly outperform
immediate labels (Theorem~\ref{thm:nonzero-delay-optimal}). A
variance bound for mixing heterogeneous-delay labelers holds under a
per-labeler sample normalization (Theorem~\ref{thm:diversity}); at a
matched total budget the predicted benefit does not appear, and
Experiment~5 reports this as a negative result.
\end{enumerate}

The construction parallels recent introductions of new combinatorial
dimensions for non-standard PAC observation
models~\citep{shaeiri2026pac,alon2022theory}. The main text is
self-contained. Every result is stated in full and comes with a proof
or a proof sketch that carries the key computation. Full proofs with
explicit constants appear in the full version of this paper. Code,
all figures, and per-run experimental records are available in a
public repository.\footnote{\url{https://github.com/krimler/delay_complexity_dimension/}}

\subsection{Related Work}
\label{sec:related}

\textbf{Delayed feedback and bandit learning with delays.}
The applied delayed-feedback literature originates with
\citet{chapelle2014modeling}. The closest precursor to our work is
\citet{yang2022generalized}, which uses conditional entropy as the
information measure in streaming recommendation.
\citet{joulani2013online} initiated the systematic study of online
learning with delays, showing that delay costs only an additive
regret penalty in stochastic problems. The line continues through
\citet{vernade2017stochastic,vernade2020non} for stochastic bandits.
Most directly, \citet{masoudian2023delayed} prove a regret rate of
$\sqrt{(K+\min\{|\cS|,d\})T}$ for delayed bandits with intermediate
observations, characterizing when those observations help. We provide
the PAC-batch analogue. The phrase ``delayed online-to-PAC'' appears
in \citet{abeles2025generalization} in a different technical sense,
as an analysis device converting online regret into generalization
bounds for mixing processes.

\textbf{Censored ERM, survival, and active learning under delay.}
For continuous $Y$ with right-censoring, \citet{ausset2022empirical}
prove finite-sample bounds of order $\sqrt{\log n/n}$ for ERM weighted
by the Kaplan--Meier estimator \citep{kaplan1958nonparametric}. Our bounds reduce to theirs when $I_W=0$.
\citet{goldberg2012q} provide analogous Q-learning results.
\citet{agarwal2013paraactive} give IWAL-style bounds for active
learning under delayed updates. They treat delay as additive latency.
Our active-expedition theorem uses the delay as an informative
signal. Related settings include temporal label
noise~\citep{nagaraj2025temporal}, decision-induced
censoring~\citep{yang2024generalization}, and the label latency
surveyed in multi-label stream
classification~\citep{multilabel2025stream}.

\textbf{New dimensions for non-standard PAC observation models.}
\citet{alon2022theory} extend PAC theory to partial concept classes
and characterize their learnability by the VC dimension of the
partial class. \citet{shaeiri2026pac} introduce the bandit
DS dimension via pseudo-boxes for multiclass PAC with bandit feedback,
with matching bounds. Our DCD follows this structural template for the
delayed-and-censored observation model. The foundation of VC
theory~\citep{vapnik1971uniform,blumer1989learnability,hanneke2016optimal}
is assumed.

\section{Setup and the Delay Complexity Dimension}
\label{sec:setup}

\textbf{Data and observation model.}
Let $(X,Y,D)\sim P$ be drawn i.i.d., with $X\in\cX$, $Y\in\{0,1\}$,
$D\in[0,\infty]$ the delay. Fix a window $W\ge 0$. The learner observes
$Z=(X,A,Y\cdot A)$ where $A=\mathbf{1}\{D\le W\}$. The label is seen
only when $A=1$. The risk of $h\in\cH$ is
$R(h)=\Prob(h(X)\ne Y)$. We measure performance by the excess risk
$R(h)-\inf_{h\in\cH}R(h)$. The realizable case $\inf_{h\in\cH}R(h)=0$
is the reference setting, and the lower-bound construction perturbs
labels by $\varepsilon$ around it. Given $n$ i.i.d.\ observations
$Z_1,\ldots,Z_n$, a learner is $(\varepsilon,\delta)$-PAC if, with
probability at least $1-\delta$, it returns $\hat h$ with excess risk
at most $\varepsilon$.

\textbf{Key quantities.}
We write $q(W)=\Prob(A=1)$ for the in-window probability,
$q(x)=\Prob(A{=}1\mid X{=}x)$ for its conditional version, and
$I_W=I(Y;A\mid X)=H(Y\mid X)-H(Y\mid X,A)$ for the (symmetric) delay
informativeness. Let $d=\mathrm{VC}(\cH)$. Entropies and mutual
informations are in nats throughout, and
$H_2(p)=-p\ln p-(1-p)\ln(1-p)$ denotes the binary entropy. The symbols
$c$, $C$, $C_1$, $C_2$ denote absolute constants that may change
between displays. The effective coverage is the exact per-sample
label-information rate of the observation channel,
\begin{equation}
\eta(W)\;:=\;\frac{I(Z;Y\mid X)}{H(Y\mid X)}\;\in\;[0,1].
\label{eq:eta}
\end{equation}
The ratio requires $H(Y\mid X)>0$, and we assume this throughout; the
lower-bound construction perturbs labels by $\varepsilon$ and
satisfies it for every $\varepsilon>0$.
The label is revealed exactly on $\{A=1\}$. The censored branch
leaves residual entropy $H(Y\mid X,A{=}0)$. Averaging the two
branches gives $H(Y\mid X,Z)=\E[(1-q(X))\,H(Y\mid X,A{=}0)]$, and we
obtain the identity
\begin{equation}
I(Z;Y\mid X)\;=\;H(Y\mid X)-\E\bigl[(1-q(X))\,H(Y\mid X,A{=}0)\bigr].
\label{eq:identity}
\end{equation}
Writing
$I^0_W:=\E_X[H(Y\mid X)-H(Y\mid X,A{=}0)]$ for the (signed) specific
information carried by the censoring event, when $q(x)\equiv q(W)$ this
yields the decomposition
\begin{equation}
\eta(W)\;=\;q(W)+\bigl(1-q(W)\bigr)\cdot\frac{I^0_W}{H(Y\mid X)}.
\label{eq:eta-decomp}
\end{equation}
The symmetric quantity mixes both branches,
$I_W=\E[q(X)I^1(X)]+\E[(1-q(X))I^0(X)]$ with
$I^a(x)=H(Y\mid x)-H(Y\mid x,A{=}a)$. Only the censored branch enters
\eqref{eq:eta-decomp}, because the labeled branch is already fully
credited through $q(W)$. The quantity $I^0_W$ can be negative under
skewed label priors, and then $\eta(W)<q(W)$: the in-window
subsample is class-biased and censored observations mislead the
learner.

\begin{definition}[Delay Complexity Dimension]
\label{def:dcd}
The Delay Complexity Dimension of $(\cH,P,W)$ is
$\mathrm{DCD}(\cH,P,W) := d/\eta(W)$.
\end{definition}

In the exactly realizable boundary case with noiseless labels,
$H(Y\mid X)=0$ and $\eta(W)$ is undefined; there we set
$\mathrm{DCD}:=d/q(W)$ by convention, matching the censored-ERM
complexity of case (ii) below.

\begin{proposition}[Limiting cases]
\label{prop:limits}
(i) If $q(W)=1$, then $\eta(W)=1$ and $\mathrm{DCD}=d$, recovering
classical VC. (ii) If $D\perp Y\mid X$, then $I^0_W=0$,
$\eta(W)=q(W)$, and $\mathrm{DCD}=d/q(W)$, recovering censored
ERM~\citep{ausset2022empirical} in the binary setting.
(iii) If $H(Y\mid X,A{=}0)=0$ wherever censoring has positive
probability (the censoring event determines the label), then
$\eta(W)=1$ and $\mathrm{DCD}=d$ regardless of $q(W)$. The delay is
then maximally informative.
\end{proposition}

\begin{proof}
All three cases follow from \eqref{eq:identity}. The censored
residual term is multiplied by $1-q=0$ in case (i). It equals
$H(Y\mid X)$ in case (ii). It equals $0$ in case (iii).
\end{proof}

\begin{remark}[Local Fano characterization]
\label{rem:fano}
The $\mathrm{DCD}$ is a distribution-dependent quantity; the
shattered-set size never exceeds $d$, and what scales as $d/\eta(W)$
is the sample complexity of the constructed instances. The proof of
Theorem~\ref{thm:dcd-lb} exhibits $2^{\Omega(d)}$
$\varepsilon$-risk-separated hypotheses whose observation laws
differ by at most $C\varepsilon^2\eta(W)$ in per-sample KL, forcing
$\Theta(\mathrm{DCD}/\varepsilon^2)$ samples.
\end{remark}

\section{Main Results}
\label{sec:results}

We now develop the theory. We first prove PAC bounds in terms of the
DCD that match on the family of the lower-bound construction
(Section~\ref{sec:bounds}). We then use the delay
value itself, through quantization (Section~\ref{sec:informative})
and through expedited labels (Section~\ref{sec:expedition}). We
close with two structural predictions and a phase diagram
(Section~\ref{sec:phenomena}).

\subsection{Matching Bounds on the Lower-Bound Family}
\label{sec:bounds}

We first show that the DCD lower-bounds the sample complexity of
every learner. We then show that an oracle pseudo-label ERM attains
the rate on the hard instances that the lower bound constructs. Let
$\iota_{\max}(q):=\bigl[\ln 2-H_2\bigl(\tfrac12
+\tfrac{\min\{q,1-q\}}{2(1-q)}\bigr)\bigr]/\ln 2$
denote the largest censored-branch information ratio an indicator
delay channel can realize at balanced labels with in-window mass $q$
(derived in the proof of Theorem~\ref{thm:dcd-lb}).

\begin{theorem}[Information-theoretic lower bound]
\label{thm:dcd-lb}
For every $q\in(0,1]$ and every censored-branch information ratio
$\iota\in[0,\iota_{\max}(q)]$ realizable by a delay channel at
balanced labels, there is a distribution $P$ with $q(W)=q$ and
$I^0_W/H(Y\mid X)=\iota$, so that $\eta(W)=q+(1-q)\iota$ and
$\mathrm{DCD}=d/\eta(W)$. Under $P$, any
$(\varepsilon,\delta)$-PAC learner needs
\[
n\;\ge\;c\,\frac{d+\log(1/\delta)}{\eta(W)\,\varepsilon^2}
\;=\;\Omega\!\bigl((\mathrm{DCD}+\log(1/\delta))/\varepsilon^2\bigr).
\]
\end{theorem}

\begin{proof}[Proof sketch]
Varshamov--Gilbert gives a packing $\Sigma\subseteq\{0,1\}^d$ of size
$\ge 2^{d/8}$ with pairwise Hamming distance $\ge d/4$. For each
$\sigma\in\Sigma$ define $P_\sigma$ with
$\Prob(Y=1\mid X=x_i)=\tfrac12+\varepsilon(2\sigma_i-1)$ on a
VC-shattered set. The delay channel
$a_y=\Prob(D\le W\mid Y{=}y,x_i)$ is held fixed across $\sigma$. Write
$\bar a=(a_1+a_0)/2$ and $\gamma_W=a_0-a_1$. The censoring
probability $\Prob_\sigma(A=0\mid x_i)=1-\bar a+\varepsilon\gamma_W$
itself depends on $\sigma_i$. The observation must therefore be
treated as a single three-outcome law,
$(A,YA)\in\{(1,1),(1,0),(0,0)\}$ with probabilities
$\bigl((\tfrac12+\varepsilon)a_1,\,(\tfrac12-\varepsilon)a_0,\,
1-\bar a+\varepsilon\gamma_W\bigr)$.
The exact three-term $\chi^2$ computation, the bound
$|\gamma_W|\le 2\min\{\bar a,1-\bar a\}$, and the entropy sandwich
$2s^2\le\ln 2-H_2(\tfrac12+s)\le 4s^2$ applied to the censored
posterior $\tfrac12+s$, $s=\gamma_W/(4(1-\bar a))$, give for
$\varepsilon\le 1/8$
\[
\mathrm{KL}_i(P_\sigma\|P_{\sigma'})\le
C\,\varepsilon^2\Bigl[\bar a(x_i)+\bigl(1-\bar a(x_i)\bigr)
\tfrac{I^0(x_i)}{H(Y\mid X)}\Bigr],\quad C\le 34.
\]
By \eqref{eq:identity}, the bracket averages over the design exactly
to $\eta(W)$. The joint KL of $n$ samples is therefore at most
$C n\varepsilon^2\eta(W)$. Fano then yields the $d$-term. The same
construction on a single point with the two-point
Bretagnolle--Huber bound \citep{tsybakov2009introduction} yields the
$\log(1/\delta)$-term. Sweeping $(a_1,a_0)$ at fixed mean $q$
realizes every $\iota\in[0,\iota_{\max}(q)]$. The full
computation appears in the full version.
\end{proof}

\begin{remark}[Achievable coverage at balanced labels]
\label{rem:balanced}
Within the construction ($p=\tfrac12$), the censored-branch ratio
$\iota(x)=I^0(x)/H(Y\mid x)$ at each point obeys
$\iota(x)\le 4s(x)^2/\ln 2$. The bound
$|\gamma_W|\le 2\min\{\bar a,1-\bar a\}$ then gives
$(1-\bar a)\,\iota(x)\le \bar a/\ln 2$ pointwise. Therefore
\[
\eta(W)\;\le\;\Bigl(1+\frac{1}{\ln 2}\Bigr)\,q(W)\;\le\;2.5\,q(W)
\]
for every delay channel with balanced labels. Coverage far
exceeding $q(W)$ is an unbalanced-label phenomenon
(Remark~\ref{rem:tightness}).
\end{remark}

The upper bounds of this paper are stated for an oracle learner. We
record the assumption once and cite it where it is used.

\begin{assumption}[Oracle pseudo-label posterior]
\label{ass:oracle}
The learner has access to the calibrated posterior
$\hat Y(x,a)=\E[Y\mid X{=}x,A{=}a]$, and in
Section~\ref{sec:informative} to its quantized version
$\hat Y(x,q)=\E[Y\mid X{=}x,Q_M{=}q]$. The label is unobserved
exactly when $A=0$, so this posterior is an oracle quantity;
estimating it from censored data is discussed in
Remark~\ref{rem:plugin} and open.
\end{assumption}

\begin{theorem}[Oracle pseudo-label ERM upper bound]
\label{thm:dcd-ub}
Under Assumption~\ref{ass:oracle}, pseudo-label ERM returns $\hat h$
with
\[
R(\hat h)-\inf_{h\in\cH}R(h)
\le C\sqrt{\tfrac{d/q(W)+\log(1/\delta)}{n}}
\]
up to logarithmic factors, with probability at least $1-\delta$.
\end{theorem}

Remark~\ref{rem:balanced} gives
$\eta(W)\le(1+1/\ln 2)\,q(W)\le 2.5\,q(W)$ for every balanced-label
delay channel. The bound therefore matches Theorem~\ref{thm:dcd-lb}
on every instance family realized by the lower-bound construction,
and the rate is $O(\sqrt{\mathrm{DCD}/n})$ there up to constants and
logarithmic factors.

\begin{proof}[Proof sketch]
Construct pseudo-labels $\hat Y(x,a)=\E[Y\mid X=x,A=a]$
(Assumption~\ref{ass:oracle}). Run ERM on
the lifted class $\cH'=\{(x,a)\mapsto h(x)\}$ with the calibrated
square-loss surrogate. The bias-variance decomposition
(Lemma~\ref{lem:biasvar}, stated with the other supporting lemmas in
Section~\ref{sec:informative}) shows that the residual variance is at
most $H(Y\mid X)-I_W$ (Lemma~\ref{lem:quantize}, via data processing).
Lemma~\ref{lem:rad} controls the lifted-class Rademacher
complexity at $O(\sqrt{d/n})$. The two combine, via the Massart-noise
calibration of \citet{massart2006risk} and the $1/q(W)$ censoring
factor of the lifted-class step, to $O(\sqrt{d/(n q(W))})$. The
constants are tracked in the full version.
\end{proof}

\begin{remark}[Tightness in the remaining regime]
\label{rem:tightness}
Under unbalanced label priors, $\eta(W)$ can far exceed $q(W)$ (the
heavily censored, informative regime), and there the upper bound is
within a factor $\eta(W)/q(W)$ of the lower bound. The
calibrated pseudo-label $\hat Y(X,A)$ is locally Fisher-efficient. Its
per-sample Fisher information matches the $\chi^2$ rate charged by the
lower bound. We therefore conjecture that the
$O(\sqrt{\mathrm{DCD}/n})$ rate holds in this regime as well. A
variance-weighted uniform-convergence argument would settle it, and
we leave this open (Section~\ref{sec:discussion}).
\end{remark}

\subsection{Delay Informativeness via Quantization}
\label{sec:informative}

We now use the delay value itself. The delay can be quantized into
$M$ bins, written $Q_M=f_M(D\wedge W)$. The information the bins
retain lowers the residual variance of pseudo-label ERM. In the
oracle setting the complexity term does not grow with $M$
(Lemma~\ref{lem:rad}); the price of finer bins is paid in
calibration (Remark~\ref{rem:plugin}).

\begin{lemma}[Quantization preserves information]
\label{lem:quantize}
For any $M\ge 2$ there is a deterministic quantizer $f_M$ such that
$I(Y;f_M(D\wedge W)\mid X)\ge I_W$.
\end{lemma}

\begin{proof}
Take any quantizer whose partition refines the in-window indicator
(the successive-refinement construction on the posterior is one such
choice). Then $A$ is a deterministic function of $Q_M=f_M(D\wedge W)$,
and data processing gives
$I(Y;Q_M\mid X)\ge I(Y;A\mid X)=I_W$.
Refining further only increases $I(Y;Q_M\mid X)$ toward
$I(Y;D\wedge W\mid X)$.
\end{proof}

\begin{lemma}[Bias-variance decomposition]
\label{lem:biasvar}
For $Q_M=f_M(D\wedge W)$ and $\hat Y(x,q)=\E[Y\mid X=x,Q_M=q]$,
$\E[(h(X)-Y)^2]=\E[(h(X)-\hat Y(X,Q_M))^2]+\E[\mathrm{Var}(Y\mid X,Q_M)]$,
and $\E[\mathrm{Var}(Y\mid X,Q_M)]\le H(Y\mid X)-I(Y;Q_M\mid X)$.
\end{lemma}

\begin{proof}
The cross term vanishes by the tower property, since
$\hat Y(X,Q_M)=\E[Y\mid X,Q_M]$. For the entropy bound,
$\mathrm{Var}(Y\mid X,Q_M)=\rho(1-\rho)\le H_2(\rho)$ with
$\rho=\Prob(Y{=}1\mid X,Q_M)$. Taking expectations and using
$H(Y\mid X,Q_M)=H(Y\mid X)-I(Y;Q_M\mid X)$ gives the claim.
\end{proof}

\begin{lemma}[Lifted-class complexity]
\label{lem:rad}
The lifted class $\cH'=\{(x,q)\mapsto h(x)\}$ satisfies
$\mathrm{VC}(\cH')=d$, and its empirical Rademacher complexity is
$O(\sqrt{d/n})$. Neither quantity depends on $M$.
\end{lemma}

\begin{proof}
Members of $\cH'$ ignore the bin argument. A set of pairs
$\{(x_i,q_i)\}$ containing a repeated $x_i$ cannot be shattered, and
a set with distinct $x_i$ is shattered by $\cH'$ exactly when
$\{x_i\}$ is shattered by $\cH$. Hence
$\mathrm{VC}(\cH')=\mathrm{VC}(\cH)=d$. Sauer--Shelah with
symmetrization (\citealt{devroye1996probabilistic}) yields the
Rademacher bound.
\end{proof}

\begin{theorem}[Quantized oracle pseudo-label ERM]
\label{thm:info-tradeoff}
Under Assumption~\ref{ass:oracle}, ERM on $\cH'$ with pseudo-labels
$\hat Y(X,Q_M)$ achieves, with probability at least $1-\delta$,
\begin{multline*}
R(\hat h_M)-\inf_{h\in\cH}R(h)\\
\le C_1\sqrt{\tfrac{d/q(W)+\log(1/\delta)}{n}}
+C_2\sqrt{\tfrac{H(Y\!\mid\!X)-I(Y;Q_M\!\mid\!X)}{n}}
\end{multline*}
up to logarithmic factors. The bound is monotone improving in the
partition: $I(Y;Q_M\mid X)\ge I_W$ for every $M\ge 2$, and refining
$f_M$ never decreases it (Lemma~\ref{lem:quantize}).
\end{theorem}

\begin{theorem}[Effective sample size]
\label{thm:neff}
With the indicator quantizer ($M=2$), oracle pseudo-label ERM
achieves excess risk $C\sqrt{(d/q(W)+\log(1/\delta))/n}$ up to
logarithmic factors, the $M=2$ case of
Theorem~\ref{thm:info-tradeoff} and the bound of
Theorem~\ref{thm:dcd-ub}. The variance-dependent term improves by
the factor $1-I_W/H(Y\mid X)$. Equivalently, the effective sample
size is $n_{\mathrm{eff}}=c\,n\,q(W)$. Informative delay never
degrades the guarantee. It reduces the residual label variance.
\end{theorem}

\begin{proof}[Proof sketch (Theorems~\ref{thm:info-tradeoff}
and~\ref{thm:neff})]
Lemma~\ref{lem:biasvar} splits the calibrated square-loss risk into
approximation plus residual variance. Lemma~\ref{lem:rad} bounds the
uniform deviations by $C\sqrt{d/n}$ for every $M$, since the lifted
class ignores the bin. The Massart calibration of
\citet{massart2006risk} converts variance into rate, and the $1/q(W)$
censoring factor enters through the lifted-class step, exactly as in
Theorem~\ref{thm:dcd-ub}. For Theorem~\ref{thm:neff}, take $M=2$. The
residual variance is then at most $H(Y\mid X)-I_W$ by
Lemma~\ref{lem:quantize}. Finer bins retain more information and
leave the complexity term unchanged, so in the oracle setting the
bound is monotone in $M$; the tradeoff that finer bins do create is
in calibration (Remark~\ref{rem:plugin}).
\end{proof}

\begin{remark}[Estimated pseudo-labels and the cost of bins]
\label{rem:plugin}
The guarantees above assume the calibrated posterior
$\hat Y(x,q)=\E[Y\mid X{=}x,Q_M{=}q]$ (Assumption~\ref{ass:oracle}).
With an estimate $\tilde Y$ in
its place, the two square-loss risks differ by at most
$2\,\E|\tilde Y-\hat Y|$ for $h,\hat Y,\tilde Y\in[0,1]$, so every
bound in this section degrades additively by
$O(\E|\tilde Y-\hat Y|)$. This additive term carries the cost of
quantization: $M$ bin posteriors must be calibrated from the
same $n$ samples, so $\E|\tilde Y-\hat Y|$ generically grows with
$M$, and an information versus calibration tradeoff appears in the
estimated setting even though the oracle bound of
Theorem~\ref{thm:info-tradeoff} is monotone. A bin-indexed class
$\{(x,q)\mapsto h_q(x)\}$ would reinstate a complexity term of order
$dM$; we do not take that route, since no bin is available at
prediction time. Estimating $\hat Y$ requires a model of the
delay channel, because $Y$ is unobserved exactly when $A=0$.
Identifiability conditions and rates for this calibration step are
open (Section~\ref{sec:discussion}).
\end{remark}

\subsection{Active Expedition}
\label{sec:expedition}

The learner may \emph{expedite} the labels of up to $B$ of its $n$
examples. Expediting an example reveals its label immediately, so the
learner observes $Z^{\exp}=(X,1,Y)$ for it. We call $B$ the query
budget.

\begin{theorem}[Effective sample size with random expedition]
\label{thm:expedition}
Let the $B$ expedited examples be drawn uniformly at random from the
$n$, independently of the observations. Under
Assumption~\ref{ass:oracle} there is a learner achieving
\[
R(\hat h)-\inf_{h\in\cH}R(h)\le C\sqrt{\tfrac{d}{n q(W)+B}+\tfrac{\log(1/\delta)}{n}}.
\]
\end{theorem}

\begin{proof}[Proof sketch]
A uniform subsample of an i.i.d.\ sample is i.i.d., so the expedited
stream consists of $B$ fully labeled draws from the original input
distribution, and the passive stream contributes effective sample
size $n\,q(W)$ by Theorem~\ref{thm:neff}. Bernstein-type
concentration on the union of the two streams
(\citealt{devroye1996probabilistic}, Ch.~12) gives the
$\sqrt{d/(n\,q(W)+B)}$ term.
\end{proof}

For the policy, let $V\subseteq\cH$ denote the version space after
the passive draws, the set of hypotheses consistent with every label
revealed by the passive sample. Fix a reference probability measure
$\mu$ on $\cH$ with $\mu(V)>0$; for finite $\cH$ one may take the
uniform distribution, and the experiments use a pilot posterior
(Remark~\ref{rem:surrogate}). We write $h\sim\mu_V$ for a draw from
the restriction of $\mu$ to $V$.

\begin{theorem}[Threshold policy]
\label{thm:policy}
Let $\Delta(x)=\Prob_{h\sim\mu_V}(h(x)=1)\,\Prob_{h\sim\mu_V}(h(x)=0)$
be the disagreement at $x$ under $\mu_V$. Consider choosing a batch
of at most $B$ expedition points, fixed before any expedited label is
revealed, to maximize the expected reduction of the $\mu$-mass of the
version space. The objective is monotone and submodular, greedy
selection achieves a $(1-1/e)$ approximation to the best fixed batch,
and the greedy marginal gain factors as
$\Delta(x)\cdot\Prob(D>W\mid x)$. The threshold rule
$\pi^\star(x)=\mathbf{1}\{\Delta(x)\Prob(D>W\mid x)\ge\tau\}$, with
$\tau$ chosen to spend the budget $B$, implements the greedy batch.
\end{theorem}

\begin{proof}[Proof sketch]
The expected-mass-reduction objective is monotone, since queries only
exclude hypotheses. It is submodular over fixed batches, since the
conditional information gain of a query is non-increasing in the
already-selected set by the chain rule. Greedy maximization under a
cardinality budget therefore achieves $(1-1/e)$ of the best fixed
batch \citep{nemhauser1978analysis}. \citet{shaeiri2026pac} use
analogous arguments. The greedy marginal gain factors as
$\Delta(x)\cdot\Prob(D>W\mid x)$, and the threshold form is its
Lagrangian dual at budget $B$. Details appear in the full version.
\end{proof}

\begin{remark}[Scope of the two guarantees]
\label{rem:scope}
Theorems~\ref{thm:expedition} and~\ref{thm:policy} answer different
questions: sample-size accounting for i.i.d.\ expedited draws, and
an approximation guarantee for the version-space objective with a
batch fixed in advance. The threshold policy's picks concentrate
where $\Delta(x)\Prob(D>W\mid x)$ is large, so they are neither
distributed as $P_X$ nor, in sequential variants that update $V$
between picks, independent, and Theorem~\ref{thm:expedition} does
not apply to them as stated. Extending it requires importance
weighting for the induced covariate shift and martingale
concentration for adaptive selection, and converting
version-space-mass reduction into excess risk requires additional
structure. Both extensions are open (Section~\ref{sec:discussion}).
Experiment~3 evaluates the threshold policy directly.
\end{remark}

\begin{remark}[Practical surrogate]
\label{rem:surrogate}
Exact version-space maintenance is intractable in high dimension, and
the experiments estimate $\Delta(x)$ with a pilot model
(Section~\ref{sec:experiments}). The greedy induction tolerates this.
If every step selects a point whose marginal gain is within $\zeta$
of the best available, the batch attains $(1-1/e)$ of the optimal
fixed-batch value minus a term of order $B\zeta$, and a
$\zeta$-accurate pilot inherits Theorem~\ref{thm:policy} up to that
term.
\end{remark}

\subsection{Structural Phenomena and a Phase Diagram}
\label{sec:phenomena}

\begin{figure*}[!t]
  \centering
  \subfloat[]{\includegraphics[height=2.4cm]{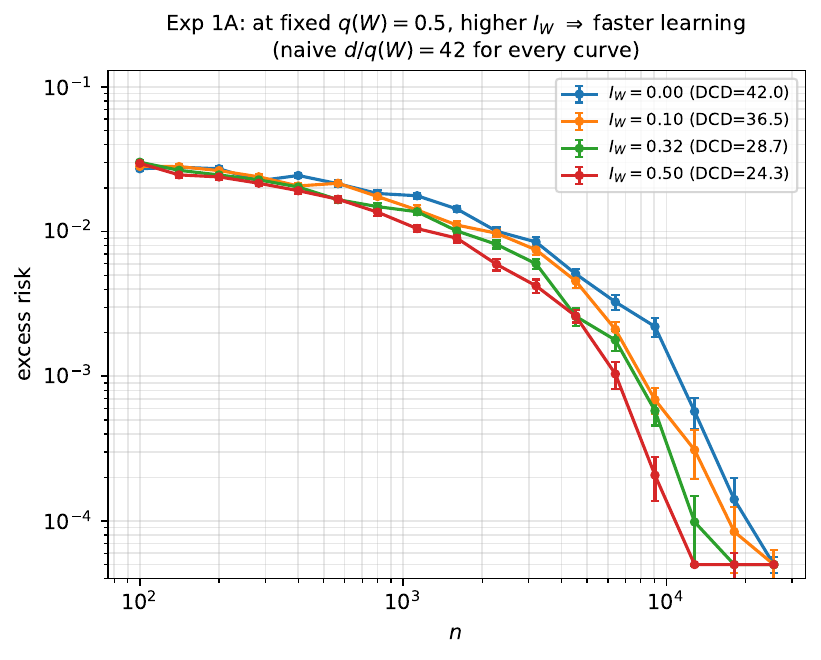}\label{fig:exp1}}\hfill
  \subfloat[]{\includegraphics[height=2.4cm]{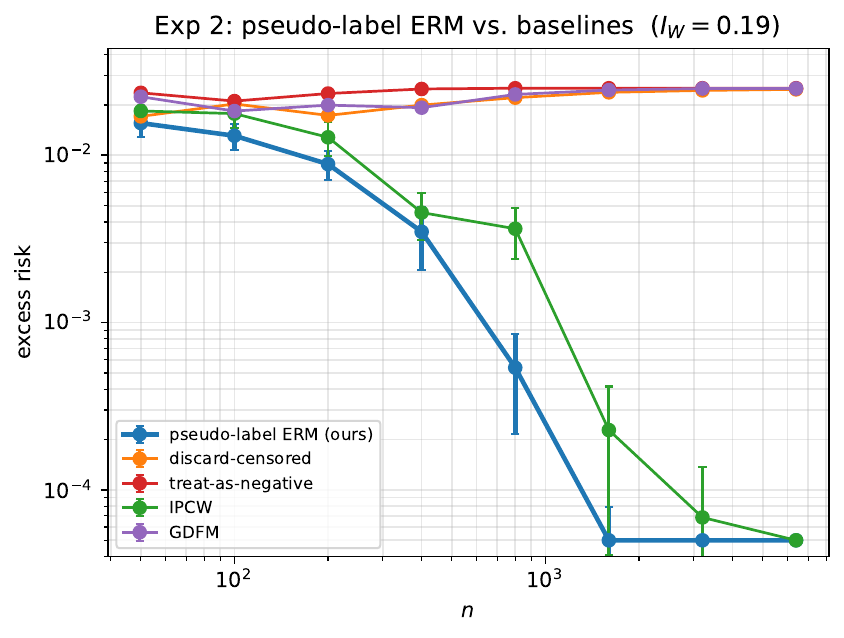}\label{fig:exp2}}\hfill
  \subfloat[]{\includegraphics[height=2.4cm]{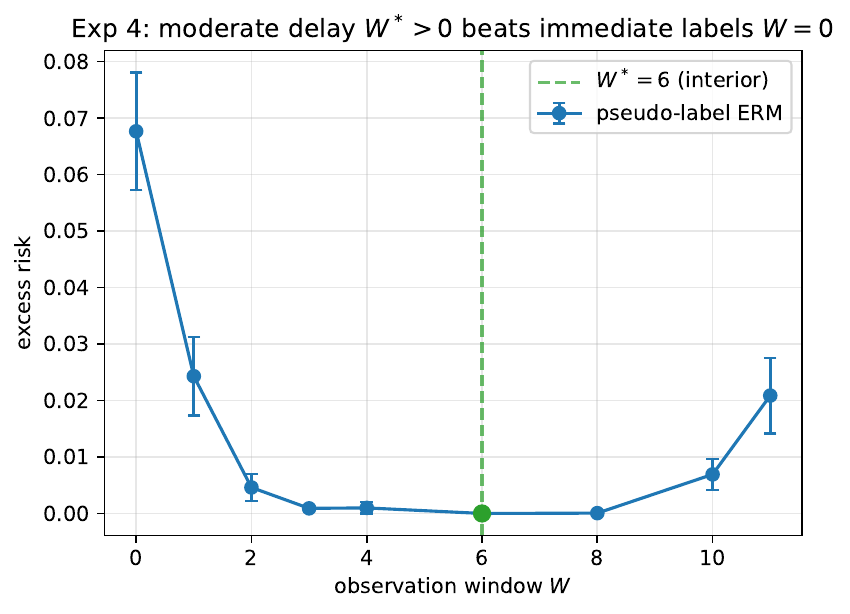}\label{fig:exp4}}\hfill
  \subfloat[]{\includegraphics[height=2.4cm]{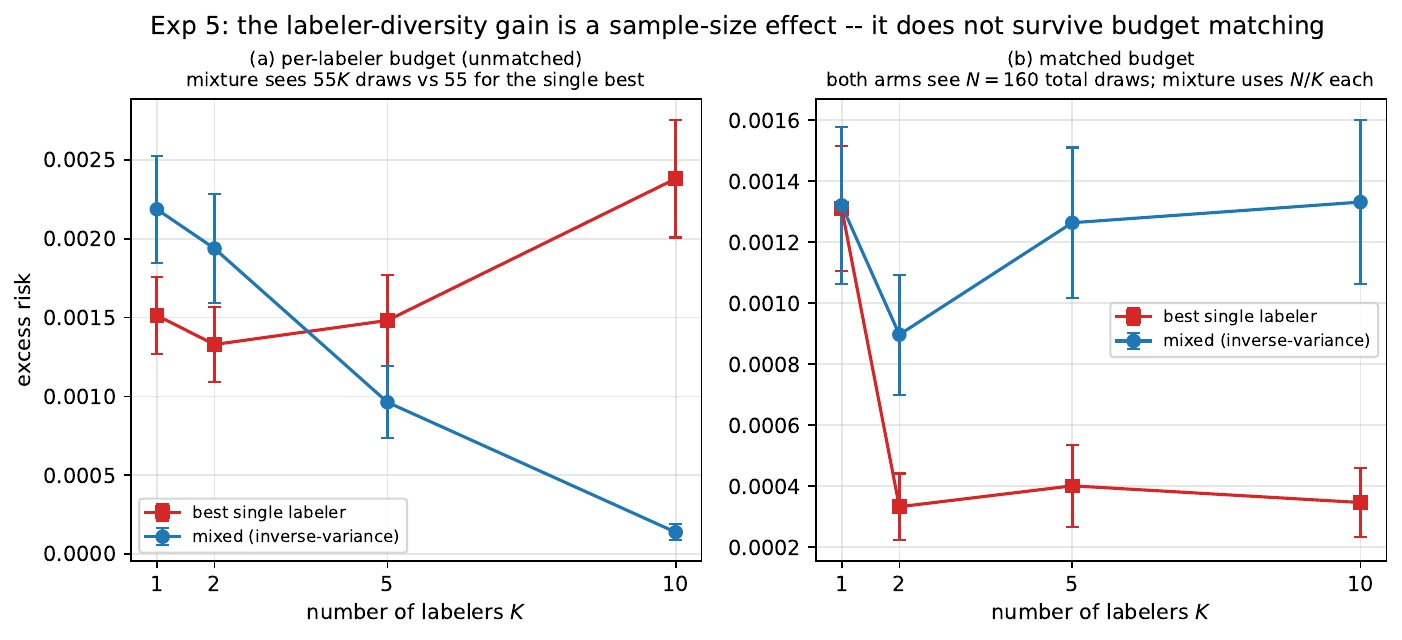}\label{fig:exp5}}\\[4pt]
  \subfloat[]{\includegraphics[height=2.4cm]{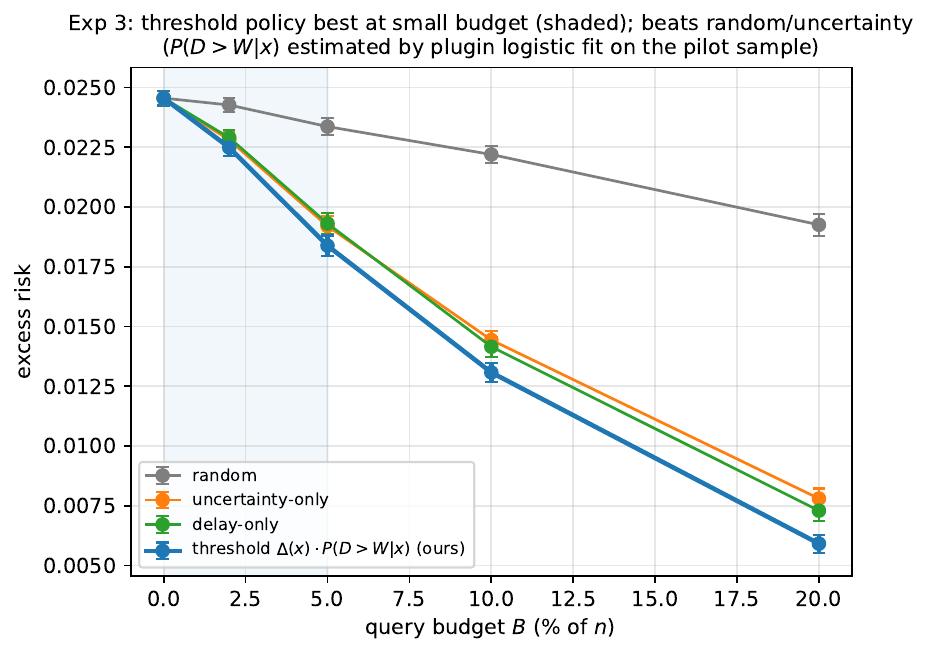}\label{fig:exp3}}\hfill
  \subfloat[]{\includegraphics[height=2.4cm]{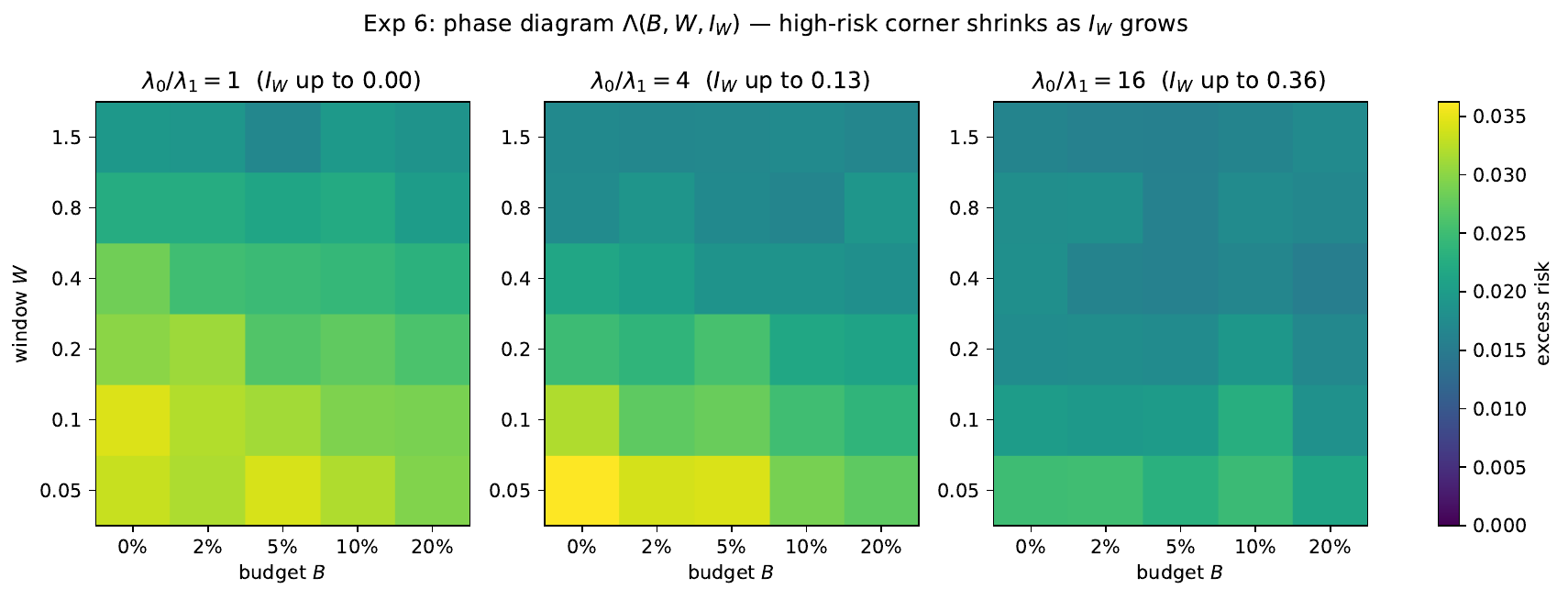}\label{fig:exp6}}\hfill
  \subfloat[]{\includegraphics[height=2.4cm]{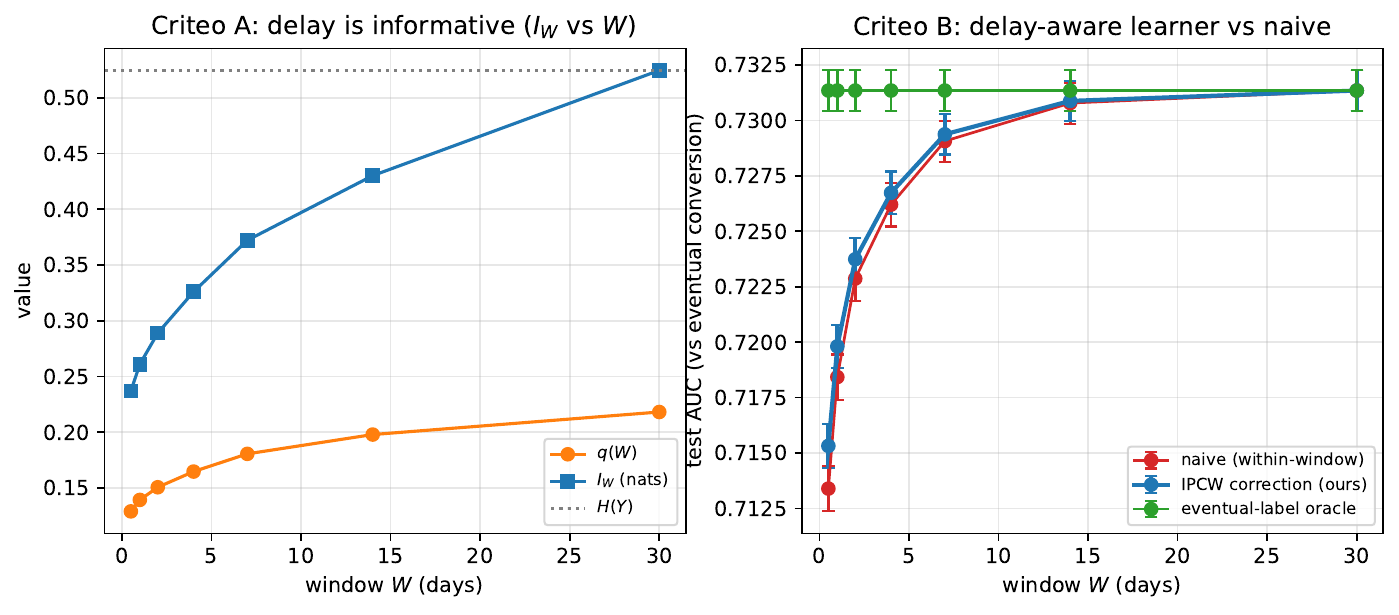}\label{fig:criteo}}
  \caption{\textbf{Experimental results.}
  (a)~Fixed coverage $q(W)=0.50$: the baseline $d/q(W)=42$ is constant and
  higher $I_W$ (smaller DCD) learns faster
  (Theorems~\ref{thm:dcd-lb},~\ref{thm:dcd-ub}).
  (b)~Censoring-aware methods are consistent; the others plateau.
  (c)~An interior optimum $W^*>0$ beats $W=0$
  (Theorem~\ref{thm:nonzero-delay-optimal}).
  (d)~Left: per-labeler budget, the mixture wins on $K$ times the data.
  Right: matched budget, the effect reverses (Theorem~\ref{thm:diversity}).
  (e)~Plug-in censoring estimate: the threshold policy is best at every
  budget (Theorems~\ref{thm:expedition},~\ref{thm:policy}).
  (f)~The starved corner shrinks as $I_W$ grows (Theorem~\ref{thm:phase}).
  (g)~Criteo: $I_W$ approaches $H(Y)$ (left); the delay-aware correction
  helps (right).}
  \label{fig:experiments}
\end{figure*}

\textbf{Nonzero-delay-optimal regime.}

We show that under a collection horizon $T$, there are regimes where
$W^\star\in(0,T)$ strictly outperforms $W=0$.

\begin{assumption}[Maturation noise with horizon]
\label{ass:maturation}
Sample size scales as $n(W)=n(1-W/T)$ for $W<T$. The label observed
at delay $W$ has Bernoulli flip rate
$\nu(W)=\nu_0 e^{-\beta W}$ with $\beta>0$,
$\nu_0\in(0,\tfrac12)$.
\end{assumption}

\begin{theorem}[Nonzero delay strictly improves]
\label{thm:nonzero-delay-optimal}
Under Assumption~\ref{ass:maturation}, the excess risk $\Phi(W)$ of
pseudo-label ERM trained at window $W$ satisfies
$\Phi(W)\le C_1\nu_0^2 e^{-2\beta W}+C_2 d/[n(1-W/T)]$. There is an
open subset of $(\beta,\nu_0,n,T,d)$ for which the minimizer
$W^\star$ lies in $(0,T)$ and $\Phi(W^\star)<\Phi(0)$.
\end{theorem}

\begin{proof}[Proof sketch]
The bias term is the squared-loss propagation of the flip rate. The
variance term is the censored sample size from
Theorem~\ref{thm:dcd-ub}. The condition
$\Phi'(0)<0$ requires $2\beta C_1\nu_0^2>C_2 d/(Tn)$, which holds for
large enough $\beta$, $\nu_0$, and $n$. Since $\Phi(W)\to\infty$ as
$W\to T^-$, the minimizer is then interior.
\end{proof}

\textbf{Diversity across labelers.}

Suppose there are $K$ independent labelers with delay rates
$\lambda_1,\ldots,\lambda_K$, and labeler $k$ has effective coverage
$\eta_k$. Here $K$ counts labelers, and it is unrelated to the bin
count $M$ of Section~\ref{sec:informative}. Write
$\bar\eta=\frac1K\sum_k\eta_k$ and $\eta_{\max}=\max_k\eta_k$.

\begin{theorem}[Diversity benefit under per-labeler normalization]
\label{thm:diversity}
Suppose each labeler contributes $n$ draws, so that
$\hat h_{\mathrm{mix}}$, which combines the $K$ labelers'
pseudo-label estimates with inverse-variance weights
$w_k\propto\eta_k$, trains on $nK$ examples, while
$\hat h_{\mathrm{best}}$, which uses only the single best labeler,
trains on $n$. Then
$\E[R(\hat h_{\mathrm{mix}})]\le
\E[R(\hat h_{\mathrm{best}})]\cdot\eta_{\max}/(\bar\eta(1+H_e(\{w_k\})))$,
where $H_e$ is the entropy of the weights.
\end{theorem}

\begin{proof}[Proof sketch]
Labeler $k$'s pseudo-label estimate has variance $O(1/(n\eta_k))$.
The inverse-variance mixture has variance
$O(1/(n\sum_k\eta_k))=O(1/(nK\bar\eta))$. The best single labeler has
variance $O(1/(n\eta_{\max}))$. Translating standard deviations to
excess risk and applying a Jensen correction for the effective number
of labelers yields the entropy factor $1+H_e(\{w_k\})$. The Jensen
step is carried out in the full version.
\end{proof}

The per-labeler normalization is part of the statement: the
mixture's gain comes from pooling $nK$ draws, and at a matched total
budget the comparison reverses empirically (Experiment~5, reported
as a negative result).

\textbf{Phase diagram.}

\begin{theorem}[Phase regimes]
\label{thm:phase}
Let $\Lambda=n\,q(W)+B$, $\Lambda_0=n\,\eta(W)+B$, $d=\mathrm{VC}(\cH)$.
(i) If $\Lambda_0=o(d)$, no consistent learner exists, and the
excess risk is $\Omega(1)$. (ii) If $\Lambda=\omega(d)$, the excess
risk is $O(\sqrt{d/\Lambda})=o(1)$; when $\Lambda=\Theta(n)$ this is
the $n^{-1/2}$ rate. (iii) If $\Lambda=\Omega(d\log n)$, the
local-Rademacher rate $\sqrt{d\log n/\Lambda}$ applies.
\end{theorem}

\begin{proof}[Proof sketch]
Case (i) is Theorem~\ref{thm:dcd-lb} in the limit $\Lambda_0/d\to 0$.
The impossibility threshold is governed by $\eta(W)$, and the
achievable threshold by $q(W)$. Case (ii) is
Theorem~\ref{thm:expedition}, whose bound is $o(1)$ exactly when
$d/\Lambda\to 0$. Case (iii) is the standard
local-Rademacher refinement under bounded variance
(\citealt{devroye1996probabilistic}), with $\Lambda$ in place of $n$.
Experiment~6 probes all three regimes.
\end{proof}

\section{Experiments}
\label{sec:experiments}


We validate the theory empirically, one prediction per experiment.
Synthetic data is the primary venue, where $d$, $q(W)$, and $I_W$
can be varied independently and computed in closed form.
Throughout, $I_W$ denotes the marginal $I(Y;A)$; the
conditional counterpart is quantified below. Curves report means
over $10$--$200$ seeds ($60$--$200$ for the rerun Experiments 1, 3,
5, and 6) with $\pm1$ standard error; compared methods run on
identical draws, so paired per-seed differences are reported.


\textbf{Synthetic observation model.}
Inputs are $x\sim\mathrm{Unif}[0,1]^k$, labels follow
$\Prob(y=1\mid x)=\sigma(\alpha\langle w^*,\tilde x\rangle)$ with
logistic $\sigma$, $h^*$ is the best-in-class halfspace, and
$d=k+1$. The delay is label-dependent,
$D\mid y\sim\mathrm{Exp}(\lambda_y)$, so the indicator
$a=\mathbf 1\{D\le W\}$ is informative, with $I_W=0$ at
$\lambda_0=\lambda_1$ and both $q(W)$ and $I_W$ in closed form.
Experiment~1 uses $k=20$, $\alpha=0.7$, $W=1.0$, and
$\lambda_0/\lambda_1\in\{1,4,16,64\}$ with $q(W)=0.50$ pinned per
condition (below). Experiments~2 and~3 use $k=12$, $\alpha=0.8$,
$W=0.6$, $(\lambda_1,\lambda_0)=(1,6)$, giving $q(W)=0.71$.
Experiment~5 uses per-labeler rates $\lambda_k\in[0.5,4]$, and
Experiment~6 the Experiment~2 geometry with $W\in[0.05,1.5]$;
remaining settings are in the run records. Excess risk is measured
against $h^*$ on $2\times10^5$ shared Monte-Carlo test points;
log-scale figures clip at a $5\times10^{-5}$ floor. All learners
share a linear backend and the true base rate $p=\tfrac12$.
Pseudo-labels use the true delay parameters, the oracle setting of
Assumption~\ref{ass:oracle}; Experiments~3 and~6 estimate the
censoring probability from data.

\textbf{Sample complexity tracks the DCD (Experiment~1).}

Theorems~\ref{thm:dcd-lb} and~\ref{thm:dcd-ub} predict that reaching
excess risk $\varepsilon$ takes $n\sim\mathrm{DCD}/\varepsilon^2$
samples. The DCD and the $I_W$-independent baseline $d/q(W)$
separate only at fixed $q(W)$, and sweeping $\lambda_0/\lambda_1$ at
a fixed window moves both. We therefore pin the coverage, solving
$p\,e^{-\lambda_1 W}+(1-p)\,e^{-r\lambda_1 W}=1-q^*$ for $\lambda_1$
at each rate ratio $r$. Every condition then sits at $q(W)=0.50$,
the baseline is $d/q(W)=42$ identically, and the DCD falls from
$42.0$ to $24.3$ (factor $1.73$) as $I_W$ rises from $0$ to $0.50$
nats, so any surviving ordering is attributable to $I_W$ alone.
Figure~\ref{fig:exp1} shows the four curves ordered by $I_W$. Over
the resolvable tolerances $\varepsilon\in[0.010,0.024]$, the sample
requirement $n^*(\varepsilon)$ is a factor $1.87$ smaller at
$I_W=0.50$ than at $I_W=0$ (geometric mean at the $\sqrt2$ grid
resolution; $3200\to1600$ at $\varepsilon=0.010$). The DCD predicts
$1.73$; the baseline predicts $1.00$ and is excluded.

\textbf{Oracle pseudo-label ERM versus baselines (Experiment~2).}

Four alternatives share the linear backend and differ only in their
treatment of censored examples: discard-censored, treat-as-negative,
class-conditional IPCW~\citep{ausset2022empirical}, and a
single-pass batch adaptation of GDFM~\citep{yang2022generalized}.
With $\lambda_0\gg\lambda_1$ the in-window sample is class-skewed,
and the methods that ignore this are biased: discard-censored,
treat-as-negative, and GDFM plateau near $0.025$ at $n=6400$
(Figure~\ref{fig:exp2}); GDFM's gains target drift and post-click
signals absent here. Pseudo-label ERM and IPCW drive excess risk to
$0$, and pseudo-label ERM dominates IPCW at every finite $n$ because
it also uses the censored examples.

\textbf{Nonzero-delay-optimal regime (Experiment~4).}

A label read after waiting $W$ is flipped with probability
$\nu_0e^{-\beta W}$, and a fixed collection horizon $T$ shrinks the
usable sample as $1-W/T$, realizing
Assumption~\ref{ass:maturation}. Figure~\ref{fig:exp4} shows the
predicted U-shape (Theorem~\ref{thm:nonzero-delay-optimal}): excess
risk falls from $0.068$ at $W=0$ to essentially $0$ at the interior
optimum $W^*=6$, then rises as the sample starves. A monotone
per-example control is in the repository.

\textbf{Diversity across labelers (Experiment~5).}

The outcome depends entirely on the budget normalization
(Figure~\ref{fig:exp5}). Under the per-labeler normalization that
Theorem~\ref{thm:diversity} states, the inverse-variance mixture
wins and the gap grows with the rate entropy: $0.0001$ against
$0.0024$ for the best single labeler at $K=10$, a
$6.0$-standard-error gap, earned on $55K$ draws against $55$. At a
matched total budget of $N=160$ draws, split $N/K$ per labeler
against all $N$ on the fastest labeler, the effect reverses: the
mixture is worse at every $K>1$, by $2.4$, $3.2$, and $3.3$ standard
errors at $K=2,5,10$, because concentrating the budget yields $138$
usable labels while spreading it over ten labelers yields $84$.
These labelers differ in rate only, so mixing buys only the
sample-size effect that matching removes, and we report the
diversity benefit as \emph{not} demonstrated at a matched budget. A
genuine test needs labelers that differ beyond latency, in label
quality or input coverage.

\textbf{Active expedition and phase diagram (Experiments~3 and~6).}

We compare the threshold policy of Theorem~\ref{thm:policy}, ranking
by $\Delta(x)\cdot\Prob(D>W\mid x)$, against random,
uncertainty-only, and delay-only selection. The censoring
probability is a plug-in: the indicator $a$ is observed for every
passive example, so $\Prob(a=0\mid x)$ is an ordinary logistic
regression fit on the pilot sample already used for $\Delta(x)$, and
the policies hold no information the baselines lack. The threshold
policy is best at \emph{every} budget (Figure~\ref{fig:exp3}),
beating random by $21$--$48$ standard errors and the informed
ablations by $5.5$--$12.5$; in the small-budget regime $B/n\le5\%$,
where the $(1-1/e)$ guarantee is meaningful, it leads
uncertainty-only by $5.5$ and delay-only by $6.4$. An
oracle-censoring control reproduces the submitted version's finding
that delay-only catches up at $B/n\ge10\%$; that catch-up is an
oracle artifact, since delay-only stakes everything on the censoring
score and degrades when it is estimated. Experiment~6 runs the full
pipeline with the same plug-in over $(B,W)$ at three
informativeness levels (Figure~\ref{fig:exp6}). The high-risk
starved corner (small $W$, small $B$) shrinks as $I_W$ grows, from
$0.0328$ to $0.0323$ to $0.0224$ averaged over the two smallest
windows and budgets; the oracle control agrees. Coverage is the dominant axis, illustrating
the regime structure of Theorem~\ref{thm:phase}.

\textbf{Criteo Conversion Logs.}

On the canonical delayed-feedback
benchmark~\citep{chapelle2014modeling} ($\sim$16M clicks, $30$-day
maximum delay), the held-out estimate
$\hat I_W=\hat H(Y)-\hat H(Y\mid A)$ rises from $0.24$ to $0.52$
nats as $W$ grows from $12$\,h to $30$ days, a large fraction of
$H(Y)\approx0.53$ (Figure~\ref{fig:criteo}, left). Censoring is
one-sided, since an in-window conversion implies $Y=1$; then, with
\eqref{eq:eta} computed marginally, $I_W=I(Z;Y)=\eta(W)\,H(Y)$, and
the curve reads directly as $\hat\eta(W)$ rising toward $1$. A naive within-window classifier
trails an eventual-label oracle by $\approx2$ AUC points at small
$W$ and converges by $30$ days, with IPCW closing part of the gap
(right; logistic backend, five 80/20 splits). The static dump carries every label, so performance is monotone in
$W$ and no U-shape appears.

\textbf{Robustness and reproducibility.}

Excess risk degrades gracefully in $W$, converges under Massart
noise with a noise-level floor, and heavy-tailed (Pareto) delays
track the exponential case (figures in the repository). A disclaimed
check on binarized survival data
(FLCHAIN~\citep{dispenzieri2012use}, death within five years),
outside our theory, reproduces the Criteo pattern with $I_W$ up to
$0.37$ nats.
All experiments are CPU-only, run in minutes, and re-run
bit-for-bit from self-contained, deterministically seeded records;
the oracle-censoring and per-labeler-budget variants remain runnable
controls. Code and
records:
\url{https://github.com/krimler/delay_complexity_dimension/}.

\textbf{The empirical DCD.}

The theory uses the conditional $I(Y;A\mid X)$; the figures report
the marginal $I_W=I(Y;A)$, the only version identified on real data
without a model of $\Prob(Y\mid X)$. Across all fourteen synthetic configurations the conditional is
uniformly the smaller, differs by at most $1.2\%$, and shifts the
DCD by at most $0.08\%$ (full table in the repository); no reported
ordering depends on the choice, and Criteo and FLCHAIN keep the
marginal on the one-sided-censoring argument above. Overlaid ``DCD
predictions'' use the operational surrogate
$d/[q(W)+(1-q(W))\,I_W/H(Y)]$, the decomposition
\eqref{eq:eta-decomp} with $I_W$ in place of $I^0_W$, which matches
the limiting cases of Proposition~\ref{prop:limits}.

\section{Discussion and Open Questions}
\label{sec:discussion}

We introduced the Delay Complexity Dimension, a PAC-theoretic
dimension for binary classification with delayed and censored
labels. It lower-bounds the sample complexity of every learner, an
oracle pseudo-label ERM matches it on the lower-bound family, and
there the DCD is within a constant of the censored-ERM complexity
$d/q(W)$ (Remark~\ref{rem:balanced}); the regime where the two
separate is the paper's main open problem. The dimension recovers
classical VC and censored-ERM rates as limiting cases and reveals a
phenomenon with no analogue in the no-delay setting: under a fixed
collection horizon, a moderate delay can strictly improve learning.
A companion diversity prediction holds under a per-labeler
normalization and did not survive a budget-matched test
(Experiment~5).

We conclude with some open questions. First, does the
$O(\sqrt{\mathrm{DCD}/n})$ rate hold under heavily unbalanced label
priors (Remark~\ref{rem:tightness})? The Fisher-efficiency evidence
there suggests it does. Second, does the DCD characterize the
agnostic rate? We conjecture the characterization survives up to
logarithmic factors under Massart-style noise. Third, what is
the right theory for heavy-tailed delays? Our experiments suggest
robustness, and a guarantee is missing. Fourth, when must the delay
channel itself be estimated, and under what conditions is the
pseudo-label calibration consistent (Remark~\ref{rem:plugin})?
Fifth, an excess-risk bound for the threshold policy, beyond the
uniform-random selection of Theorem~\ref{thm:expedition}, is open
(Remark~\ref{rem:scope}). Finally, the contextual bandit version
remains open and would connect to \citet{masoudian2023delayed}.

\textbf{Acknowledgment.}
Portions of the text of this paper were edited with the
assistance of the AI system Claude (Anthropic). All technical
content, proofs, experiments, and conclusions are the author's own,
and all AI-assisted text was reviewed and verified by the author.

\balance
\bibliographystyle{IEEEtranSN}
\bibliography{refs}

\end{document}